\documentclass[twoside,11pt]{article}

\newif\ifmydraft
\mydrafttrue
\mydraftfalse

\newif\ifanon
\anontrue
\anonfalse

\newif\ifarXiv
\arXivtrue

\usepackage[preprint]{jmlr2e}
\jmlrheading{1}{2000}{1-48}{4/00}{10/00}{meila00a}{Marina Meil\u{a} and Michael I. Jordan}

\firstpageno{1}

\ifmydraft
  \newcommand{\draftcolor}{purple}
  \newcommand{\draftcolorpage}{black!20}
\else
  \newcommand{\draftcolor}{black}
  \newcommand{\draftcolorpage}{white}
\fi

\usepackage[utf8]{inputenc} 
\usepackage[T1]{fontenc}    
\usepackage{amsfonts, amsmath}
\usepackage{microtype}      
\usepackage{graphicx}
\usepackage{xcolor}
\ifmydraft\usepackage{showlabels}\else\fi
\usepackage{comment}
\usepackage{caption}
\usepackage{subcaption}
\usepackage{tikz}
  \usetikzlibrary{arrows.meta}
  \usetikzlibrary{positioning}
  \tikzset{
    included node/.style={circle, draw=black!100, thick, on grid, minimum width=0.5cm}, 
    hidden node/.style={circle, draw=black!30, thick, on grid, minimum width=0.5cm},
    included connection/.style={->, thick, draw=black!100},
    hidden connection/.style={->, thick, draw=black!30, draw opacity=0},
    fused connection/.style={->, thick, draw=black!100},
    fidden connection/.style={->, thick, draw=black!30, draw opacity=0},
    node title/.style={above=0.8cm of five-one, font=\bfseries},
    general/.style={node distance=1cm and 0.7cm}
}

\pagecolor{\draftcolorpage}
\ifmydraft
  \newcommand{\jl}[1]{\color{blue}#1\color{black}}
\else
  \newcommand{\jl}[1]{}
\fi

\DeclareMathOperator*{\argmin}{arg\,min}

\makeatletter 
\newcommand*{\deq}{\ensuremath{\mathrel{\rlap{%
\raisebox{0.3ex}{$\m@th\cdot$}}%
\raisebox{-0.3ex}{$\m@th\cdot$}}=}}
\makeatother

\newcommand{\zero}{\textcolor{\draftcolor}{\boldsymbol{0}}}

\newcommand{\abs}[1]{\lvert#1\rvert}
\newcommand{\absB}[1]{\bigl\lvert#1\bigr\rvert}
\newcommand{\absBB}[1]{\Bigl\lvert#1\Bigr\rvert}
\newcommand{\absBBB}[1]{\biggl\lvert#1\biggr\rvert}
\newcommand{\absBBBB}[1]{\Biggl\lvert#1\Biggr\rvert}

\newcommand{\norm}[1]{|\!|#1|\!|}

\newcommand{\normM}[1]{|\!|\!|#1|\!|\!|}

\newcommand{\normtwo}[1]{\norm{#1}_2}

\newcommand{\normtwos}[1]{\norm{#1}_2^2}

\newcommand{\normF}[1]{\normM{#1}_{\operatorname{F}}}

\newcommand{\nbrlayers}{\textcolor{\draftcolor}{\ensuremath{l}}}
\newcommand{\nbrsamples}{\textcolor{\draftcolor}{\ensuremath{n}}}
\newcommand{\nbrinput}{\textcolor{\draftcolor}{\ensuremath{d}}}
\newcommand{\nbrwidth}{\textcolor{\draftcolor}{\ensuremath{w}}}

\newcommand{\nbrparameter}{\textcolor{\draftcolor}{\ensuremath{p}}}
\newcommand{\nbrparameterj}{\textcolor{\draftcolor}{\ensuremath{\nbrparameter^j}}}
\newcommand{\nbrparameterjj}{\textcolor{\draftcolor}{\ensuremath{\nbrparameter^{j+1}}}}
\newcommand{\nbrparameterl}{\textcolor{\draftcolor}{\ensuremath{\nbrparameter^{\nbrlayers}}}}

\newcommand{\parameter}{\textcolor{\draftcolor}{\ensuremath{\boldsymbol{\Theta}}}}
\newcommand{\parameterE}{\textcolor{\draftcolor}{\ensuremath{\Theta}}}
\newcommand{\parameterEj}{\textcolor{\draftcolor}{\ensuremath{\Theta^j}}}

\newcommand{\parameterEl}{\textcolor{\draftcolor}{\ensuremath{\Theta^{\nbrlayers}}}}

\newcommand{\parameterG}{\textcolor{\draftcolor}{\ensuremath{\boldsymbol{\Gamma}}}}

\newcommand{\parameterO}{\textcolor{\draftcolor}{\ensuremath{\boldsymbol{\Theta}^*}}}

\newcommand{\parameterGE}{\textcolor{\draftcolor}{\ensuremath{\Gamma}}}

 \newcommand{\parameterSFull}{\textcolor{\draftcolor}{\ensuremath{\overline{\mathcal{M}}}}}

 \newcommand{\estimator}{\textcolor{\draftcolor}{\ensuremath{\widehat{\parameter}}}}

\newcommand{\outputE}{\color{\draftcolor}\ensuremath{y}\color{black}}
\newcommand{\outputi}{\color{\draftcolor}\ensuremath{\outputE_i}\color{black}}

\newcommand{\loss}{\ensuremath{\textcolor{\draftcolor}{\ensuremath{\mathfrak{h}}}}}
\newcommand{\lossF}[1]{\textcolor{\draftcolor}{\ensuremath{\loss[#1]}}}
\newcommand{\lossFB}[1]{\textcolor{\draftcolor}{\ensuremath{\loss\bigl[#1\bigr]}}}

\newcommand{\functionS}{\ensuremath{\textcolor{\draftcolor}{\mathcal F}}}
\newcommand{\function}{\textcolor{\draftcolor}{\mathfrak f}}
\newcommand{\functionF}[1]{\textcolor{\draftcolor}{\function[#1]}}
\newcommand{\functionG}{\textcolor{\draftcolor}{\mathfrak g}}
\newcommand{\functionGF}[1]{\textcolor{\draftcolor}{\functionG[#1]}}
\newcommand{\functionT}{\ensuremath{\textcolor{\draftcolor}{\mathfrak f^*}}}
\newcommand{\functionTF}[1]{\textcolor{\draftcolor}{\functionT[#1]}}

\newcommand{\emppro}{\ensuremath{\textcolor{\draftcolor}{z}}}

\newcommand{\constlip}{\ensuremath{\textcolor{\draftcolor}{c_{\loss}}}}

\newcommand{\constenv}{\ensuremath{\textcolor{\draftcolor}{w_{\functionS}}}}
\newcommand{\constnoise}{\ensuremath{\textcolor{\draftcolor}{s_{\outputE\mid\inputv}}}}

\newcommand{\constinput}{\ensuremath{\textcolor{\draftcolor}{s_{\inputv}}}}

\newcommand{\radecomplex}{\ensuremath{\textcolor{\draftcolor}{c_{\functionS}}}}

\newcommand{\constgeneric}{\ensuremath{\textcolor{\draftcolor}{a}}}

\newcommand{\radrv}{\ensuremath{\textcolor{\draftcolor}{r}}}
\newcommand{\radrvi}{\ensuremath{\textcolor{\draftcolor}{r_i}}}

\newcommand{\level}{\ensuremath{\textcolor{\draftcolor}{t}}}

\newcommand{\erm}{\ensuremath{\textcolor{\draftcolor}{\widehat{\function}}}}
\newcommand{\ermF}[1]{\ensuremath{\textcolor{\draftcolor}{\erm[#1]}}}

\newcommand{\fracNN}[2]{#1/(#2)}

\newcommand{\parameterS}{\textcolor{\draftcolor}{\ensuremath{\mathcal{M}}}}
\newcommand{\activation}{\textcolor{\draftcolor}{\ensuremath{\mathfrak{a}}}}

\newcommand{\radius}{\ensuremath{\textcolor{\draftcolor}{b_{\parameterS}}}}
\newcommand{\radiusb}{\ensuremath{\textcolor{\draftcolor}{(b_{\parameterS})}}}

\newcommand{\inputvE}{\textcolor{\draftcolor}{\ensuremath{x}}}
\newcommand{\inputv}{\textcolor{\draftcolor}{\ensuremath{\boldsymbol{\inputvE}}}}
\newcommand{\inputvi}{\textcolor{\draftcolor}{\ensuremath{\inputv_i}}}

\newcommand{\R}{\textcolor{\draftcolor}{\mathbb{R}}}

\newcommand{\corruptionlevel}{\textcolor{\draftcolor}{c}}
\newcommand{\corruptionprob}{\textcolor{\draftcolor}{P_{\operatorname{cor}}}}
\newcommand{\corruptionvar}{\textcolor{\draftcolor}{\inputvE_{\operatorname{cor}}}}

\newcommand{\figureloss}[1]{
\begin{figure}[#1]
  \centering
     \begin{subfigure}[b]{0.305\textwidth}
\begin{tikzpicture}
   \draw[->, >=Latex] (-2.2, 0) -- (2.2, 0);
   \draw[->, >=Latex] (0, -0.2) -- (0, 2.2);

   \node at (0.4, 2.1) {\small $\lossF{a}$};
   \node at (2.1, -0.2) {\small $a$};

   \draw[thick] (-2, 2) -- (0,0) -- (2, 2);

\end{tikzpicture}  
  \caption{the absolute-deviation loss $\lossF{a}\deq\abs{a}$ is convex but not differentiable, and it satisfies the Lipschitz condition with $\constlip=1$\\\phantom{nothing}\\}
\end{subfigure}~~~~
     \begin{subfigure}[b]{0.305\textwidth}
\begin{tikzpicture}
   \draw[->, >=Latex] (-2.2, 0) -- (2.2, 0);
   \draw[->, >=Latex] (0, -0.2) -- (0, 2.2);

   \node at (0.4, 2.1) {\small $\lossF{a}$};
   \node at (2.1, -0.2) {\small $a$};

  \draw[scale=1, domain=-1:1, smooth, variable=\x, thick] plot ({\x}, {\x*\x/2});
  \draw[thick] (-2, 1.5) -- (-1, 0.5);
  \draw[thick] (1, 0.5) -- (2, 1.5);

\end{tikzpicture}  
  \caption{the Huber loss $\lossF{a}\deq a^2/2$ for $a\in[-k,k]$ and   $\lossF{a}\deq k\abs{a}-k^2/2$ otherwise, where $k\in(0,\infty)$, is convex and differentiable, and it satisfies the Lipschitz condition with $\constlip=k$}
\end{subfigure}~~~~
     \begin{subfigure}[b]{0.305\textwidth}
\begin{tikzpicture}
   \draw[->, >=Latex] (-2.2, 0) -- (2.2, 0);
   \draw[->, >=Latex] (0, -0.2) -- (0, 2.2);

   \node at (0.4, 2.1) {\small $\lossF{a}$};
   \node at (2.1, -0.2) {\small $a$};

     \draw[scale=1, domain=-2:2, smooth, variable=\x, thick] plot ({\x}, {ln(1+\x*\x/0.1)/2});

\end{tikzpicture}  
  \caption{the Cauchy loss  $\lossF{a}\deq \log[1+k^2 a^2]$, where $k\in(0,\infty)$, is not convex but differentiable, and it satisfies the Lipschitz condition with $\constlip=k$\\}
\end{subfigure}
\caption{three robust alternatives to the least-squares loss $\lossF{a}\deq a^2/2$}
  \label{fig:losses}
\end{figure}
}

\begin{document}

\title{Risk Bounds for Robust Deep Learning}

\ifanon
\else
  \author{\name Johannes Lederer \email johannes.lederer@rub.de\\
              \addr Department of Mathematics\\
              Ruhr-University Bochum\\
              Germany}

\ShortHeadings{Robust Deep Learning}{Johannes Lederer}
\fi

\editor{?}

\maketitle

\begin{abstract}
It has been observed that certain loss functions can  render deep-learning pipelines robust against flaws in the data.
In this paper, we support these empirical findings  with statistical theory.
We especially show that empirical-risk minimization with unbounded, Lipschitz-continuous loss functions,
such as the least-absolute deviation loss, Huber loss, Cauchy loss, and Tukey's biweight loss, can provide efficient prediction under minimal assumptions on the data.
More generally speaking,
our paper provides theoretical evidence for the benefits of robust loss functions in deep learning.
\end{abstract}


\section{Introduction}
Deep learning often uses data that are rich in terms of quantity but meager in terms of quality.
A well-studied problem is adversarial attacks,
which means that parts of the data are corrupted by a ``mean-spirited opponent.'' 
It has been shown that adversarial attacks can make standard deep-learning pipelines fail completely \citep{akhtar2018threat,8611298, kurakin2016adversarial, wang2019direct, 10.1145/2976749.2978392, tksl, kurakin2016adversarial2},
and a number of approaches to address this problem have been proposed \citep{aleks2017deep, kos2017delving, papernot2015distillation, tramr2017ensemble, salman2019provably,wang2018robust}.

But statistical theory for deep-learning under adversarial attacks is scarce,
and, more importantly,
there are other, arguably more common,  types of problems with the data.
For example,
data collection is often automated,
and the sheer size of typical data sets makes it difficult to uphold high data quality.
Moreover,
data are often convenience samples,
that is, the strategy for collecting  data is not necessarily appropriate for the specific purpose of the analysis. 
Thus, 
we are interested in deep learning that  caters to a broad spectrum of data in general.
We call this topic  ``robust deep learning.''

Robust learning is a classical topic in statistics \citep{Stigler2010}.
It has especially been shown that many standard estimators can be rendered robust with respect to heavy-tailed data by replacing their loss-functions, such as least-squares, by Lipschitz-continuous alternatives, such as Huber loss \citep{Hampel11,Huber09}.
The robustness-yielding properties of such loss functions have also been observed in a variety of deep-learning applications~\citep{Barron19,Belagiannis15,Jiang18,Wang16}.
But statistical theories for deep learning are restricted to bounded loss functions or presume (sub-)Gaussian or bounded input and output~\citep{Bartlett1998,Hieber2017,Taheri20}.

In this paper,
we establish a statistical theory for deep learning with  Lipschitz-continuous loss functions,
such as Tukey's biweight loss, Huber loss, and absolute-deviation loss.
We first establish a general risk bound that caters to empirical-risk minimizers with unbounded, Lipschitz-continuous loss functions.
This result might be of independent interest.
We then use the general risk bound to derive statistical guarantees for robust deep learning in a general class of feedforward neural networks.
Broadly speaking,
our theories suggest that robust loss function can lead to effective learning with problematic as well as with benign data.

\paragraph{Outline of the paper}
In Section~\ref{sec:riskbound},
we establish a general risk bound that allows for Lipschitz-continuous but unbounded loss functions.
In Section~\ref{sec:deeplearning},
we specify the risk bound in the case of weight decay with robust loss functions,
which leads to the advertised robust guarantees.
In Section~\ref{sec:proofs},
we give detailed proofs.
In Section~\ref{sec:discussion},
we briefly discuss some extensions and limitations.

\section{General Risk Bound}\label{sec:riskbound}
In this section,
we establish a risk bound that is tailored to our needs in deep learning but might also be of independent interest.
The bound is formulated in terms of the empirical risk and the Rademacher complexity and, therefore, is related to existing bounds in empirical-risk minimization.
Our key innovation is that we allow for unbounded loss functions.

We first formulate the data and functions on these data.
Consider i.i.d.\@ distributed pairs $(\outputE,\inputv),(\outputE_1,\inputv_1),\dots,(\outputE_{\nbrsamples},\inputv_{\nbrsamples})\in\R\times\R^{\nbrinput}$ 
and i.i.d.~Rademacher random variables $\radrv_1,\dots,\radrv_{\nbrsamples}\in\{\pm 1\}$. 
Also, consider a nonempty set~$\functionS$ that consists of functions of the form $\function\,:\,\R^{\nbrinput}\to\R$.
We summarize the properties of the data and the functions in four quantities:
\begin{definition}[Complexity measures]\label{complexitymeasures}
Given a function~$\functionT\in\functionS$, 
we call 
\begin{equation*}
\constinput\deq \sqrt{E_{(\outputE,\inputv)}\bigl[\normtwo{\inputv}^2\bigr]}~~~\text{and}~~~\constnoise\deq \sqrt{E_{(\outputE,\inputv)}\Bigl[\absB{\outputE-\functionTF{\inputv}}^2\Bigr]}
\end{equation*}
the \emph{expected size of the input} and the \emph{expected size the noise}, respectively,
\begin{equation*}
\constenv\deq  \sqrt{E_{(\outputE,\inputv)}\biggl[\sup_{\function\in\functionS}\absB{\functionF{\inputv}-\functionTF{\inputv}}^2\biggr]}
\end{equation*}
the \emph{size of (an envelope of) \functionS},
and
\begin{equation*}
\radecomplex\deq  E_{(\outputE_1,\inputv_1),\dots,(\outputE_{\nbrsamples},\inputv_{\nbrsamples}),\radrv_1,\dots,\radrv_{\nbrsamples}}\Biggl[\sup_{\function\in\functionS}\absBBB{\frac{1}{\nbrsamples}\sum_{i=1}^{\nbrsamples}\radrvi\functionF{\inputvi}}\Biggr]
\end{equation*}
the \emph{Rademacher complexity of~\functionS}.
\end{definition}
\noindent 
The function~\functionT\ can be an arbitrary element of~\functionS,
but we will later think of it as the ``true'' data-generating function or an approximation of it.
It then makes sense to call the quantity~$\outputE-\functionTF{\inputv}$ the ``noise.''
The quantity~\constenv\ is the size of an envelope of $\functionF{\inputv}-\functionTF{\inputv}$ over~\functionS\ \citep[Section~2]{Lederer14}.
The Rademacher complexity~\radecomplex\ is finally a well-known measure of the complexity of the set~\functionS\ \citep{Bartlett02b,Koltchinskii01,Koltchinskii02}.

We then formulate the empirical-risk minimizer.
Consider a function $\loss\,:\,\R\to\R$ that is Lipschitz continuous:
there is a constant~$\constlip\in[0,\infty)$ such that
\begin{equation}\label{Lipschitz}
  \absB{\lossF{a}-\lossF{b}}\leq\constlip\abs{a-b}~~~~~~~~\text{for all}~a,b\in\R\,.
\end{equation}
We also assume, without loss of generality, that $\lossF{0}=0$.
We call $\loss$ the \emph{loss function}.
The least-squares loss does not satisfy the Lipschitz condition, but many robust versions of it do,
including  
the absolute-deviation loss, the Huber loss, the Cauchy loss, and Tukey's biweight loss;
in particular, we do not require the loss to be convex or differentiable
(see Figure~\ref{fig:losses} for illustrations).
The empirical-risk minimizers are then
\begin{equation*}
  \erm\in\argmin_{\function\in\functionS}\Biggl\{\frac{1}{\nbrsamples}\sum_{i=1}^{\nbrsamples}\lossFB{\outputE_i-\functionF{\inputvi}}\Biggr\}\,.
\end{equation*}
We give examples of these estimators in the following section.

\figureloss{t}

We now equip the empirical-risk minimizer with a statistical guarantee:
\begin{theorem}[General risk bound]\label{riskbound}
  For every $\function\in\functionS$ and $\level\in(0,1)$,
it holds with probability at least $1-\level$ that
  \begin{equation*}
    E_{(\outputE,\inputv)}\Bigl[\loss\bigl[\outputE-\functionF{\inputv}\bigr]\Bigr]\leq \frac{1}{\nbrsamples}\sum_{i=1}^{\nbrsamples}\lossFB{\outputE_i-\functionF{\inputvi}}+16\constlip\radecomplex+236\constlip\frac{\constenv+\constnoise}{\sqrt{\nbrsamples}\level}\,.
  \end{equation*}
\end{theorem}
\noindent This inequality bounds the risk of a function~$\function$ 
in terms of its empirical loss and the complexity of the setting.
As long as the complexity terms are small enough,
and the empirical risk of the true data-generating function converges sufficiently fast to its expectation,
the above-stated inequality ensures that the population risk of the empirical-risk minimizer~\erm\ is not much larger than the population risk of the true data-generating function:
\begin{corollary}[General risk bound for~\erm]\label{riskcor} 
    For every $\level\in(0,1)$,
it holds with probability at least $1-\level$ that
  \begin{multline*}
    E_{(\outputE,\inputv)}\Bigl[\loss\bigl[\outputE-\ermF{\inputv}\bigr]\Bigr]\leq E_{(\outputE,\inputv)}\Bigl[\loss\bigl[\outputE-\functionTF{\inputv}\bigr]\Bigr]\\
+\frac{1}{\nbrsamples}\sum_{i=1}^{\nbrsamples}\biggl(\lossFB{\outputE_i-\functionTF{\inputvi}}-E_{(\outputE,\inputv)}\Bigl[\loss\bigl[\outputE-\functionTF{\inputv}\bigr]\Bigr]\biggr)+16\constlip\radecomplex+236\constlip\frac{\constenv+\constnoise}{\sqrt{\nbrsamples}\level}\,.
  \end{multline*}
\end{corollary}
We will use these results in the following section to derive risk bounds for robust deep learning.

The bound in Theorem~\ref{riskbound} is similar to the one in~\citet[Theorem~8]{Bartlett02}.
The crucial difference is that their bound requires that the loss function has values only in~$[0,1]$,
while our bound allows for loss functions that are Lipschitz continuous but unbounded.
The price for this change in scope is the inclusion of the quantities~\constenv\ and~\constnoise,
which are additional measures for the complexity of the statistical framework.

Moving from bounded to unbounded loss functions also  requires changing the proof techniques.
For example, 
proofs in the bounded case can use McDiarmid's inequality~\citep{McDiarmid89}---see, for example, \citet[Proof of Theorem~8]{Bartlett02} and \citet[Proof of Theorem~3.3.]{Mohri18}.
We instead use a concentration inequality for heavy-tailed data from~\citet{Lederer14}.
The proof is deferred to Section~\ref{sec:proofriskbound}.

We finally mention the fact that by applying the results of \citet{Lederer14} in a slightly different way,
 one can relax the assumptions on the data from a second-moment condition to a $(1+b)$th-moment condition, $b>0$, at the price of getting a slower rate;
we omit the details to avoid digression.

\section{Guarantees for Robust Deep Learning}\label{sec:deeplearning}
We now use the above-stated risk bound to develop guarantees for robust deep learning.
We consider layered, feedforward neural networks, that is,
we consider 
$\functionS\deq\{\function_{\parameter}\, :\, \R^{\nbrinput}\,\to\,\R\ :\ \parameter\in\parameterS\}$ with 
$\parameterS$ a nonempty subset of $\parameterSFull\deq\{\parameter=(\parameterEl,\dots,\parameterE^0)\ :\ \parameterE^j\in\R^{\nbrparameterjj\times\nbrparameterj}\}$  and
\begin{equation}\label{networks}
  \function_{\parameter}[\inputv]\deq\parameterE^{\nbrlayers}\activation^{\nbrlayers}\bigl[\parameterE^{\nbrlayers-1}\cdots \activation^1[\parameterE^0\inputv]\bigr]~~~~~~~~~\text{for}~\inputv\in\R^{\nbrinput}\,.
\end{equation}
The functions $\activation^j\,:\,\R^{\nbrparameterj}\to\R^{\nbrparameterj}$ are called the activation functions,
$\nbrlayers$~the depth of the network,
$\nbrparameter^0\deq\nbrinput$ and~$\nbrparameter^{\nbrlayers+1}\deq1$ the input and output dimensions, respectively,
and $\nbrwidth\deq \max\{\nbrparameter^1,\dots,\nbrparameterl\}$ the width of the network. 
To fix ideas,
we assume the popular and well-established ReLU activation: 
$(\activation^j[\boldsymbol{v}])_i\deq \max\{0,v_i\}$  \citep{10.1016/S0893-6080(98)00012-4,Salinas11956}.

The empirical-risk minimizers are then the functions 
\begin{equation}\label{erm}
 \erm\deq\function_{\estimator}\text{~~with~~}\estimator\in\argmin_{\parameter\in\parameterS}\Biggl\{\frac{1}{\nbrsamples}\sum_{i=1}^{\nbrsamples}\lossFB{\outputE_i-\function_{\parameter}[\inputvi]}\Biggr\}\,.
\end{equation}
The parameter set is assumed to satisfy\label{setweight}
\begin{equation*}
  \parameterS\subset\Bigl\{\parameter\in\parameterSFull\ :\ \max_{j\in\{0,\dots,\nbrlayers\}}\normF{\parameterEj}\leq \radius\Bigr\}
\end{equation*}
for a fixed $\radius\in[0,\infty)$
and the Frobenius norm
\begin{equation*}
   \normF{\parameterEj}\deq\sqrt{\sum_{i=1}^{\nbrparameterjj}\sum_{k=1}^{\nbrparameterj}\abs{(\parameterEj)_{ik}}^2}~~~~~~~~~~\text{for}~j\in\{0,\dots,\nbrlayers\},\,\parameterE^j\in\R^{\nbrparameterjj\times\nbrparameterj}\,.
\end{equation*}
Such choices of~\parameterS\ have been popular for more than three decades already and are known under the name ``weight decay'' \citep{Krogh92}.

A standard question is how the empirical-risk minimizers compare with an oracle.
If the model is correct,
the oracle is typically the true data-generating function;
otherwise,
the oracle is an approximation of it. 
We do not need to know the specifics:
our theory works for every oracle $\functionT\deq\function_{\parameterO}$ with a fixed~$\parameterO\in\parameterS$.
But, in any case, we can interpret~\functionT\ as the ``best'' neural network.

Common loss functions for classification, 
such as the logistic sigmoid function,
are bounded.
Statistical guarantees for corresponding empirical-risk minimizers can then be derived based on well-established risk bounds, 
such as~\cite[Theorem~8]{Bartlett02}.
Common loss functions for regression-type tasks,
in contrast,
 are unbounded.
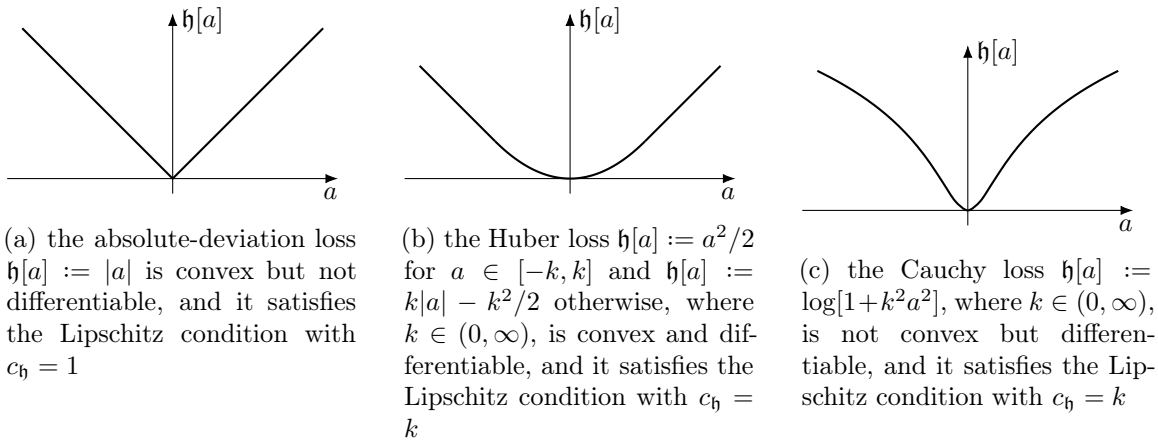
Particularly interesting for us are Lipschitz-continuous alternatives to the least-squares loss $\lossF{a}\deq a^2$.
 A basis for deriving statistical guarantees is then  Theorem~\ref{riskbound}.
Indeed, we find the following result:
\begin{theorem}[Robust deep learning]\label{robustbound}
For every $\level\in(0,1/2)$,
it holds with probability at least $1-\level$ that
  \begin{equation*}
    E_{(\outputE,\inputv)}\Bigl[\loss\bigl[\outputE-\ermF{\inputv}\bigr]\Bigr]\leq E_{(\outputE,\inputv)}\Bigl[\loss\bigl[\outputE-\functionTF{\inputv}\bigr]\Bigr]+\constgeneric\constlip\frac{\radiusb^{\nbrlayers+1}(\nbrlayers+1)\constinput+\constnoise}{\sqrt{\nbrsamples}\level}\,,
  \end{equation*}
 where $\constgeneric\in(0,\infty)$ is a numerical constant.

For every $\level\in(0,1/2)$ and $\nbrsamples$ large enough,
it holds with probability at least $1-\level$ that 
  \begin{equation*}
    E_{(\outputE,\inputv)}\Bigl[\loss\bigl[\outputE-\ermF{\inputv}\bigr]\Bigr]\leq 1.1\constlip\constnoise+\constgeneric\constlip\radiusb^{\nbrlayers+1}(\nbrlayers+1)\constinput\sqrt{\frac{\log[\nbrsamples]}{\nbrsamples}}\,.
  \end{equation*}
\end{theorem}
\noindent 
Broadly speaking, the first part of the theorem guarantees that the empirical-risk minimizers perform essentially as well as the best network in the class under consideration;
the second part of the theorem guarantees that the expected error of the empirical-risk minimizers is essentially proportional to the variance of the noise.
The key feature of the theorem is that it only requires a Lipschitz-continuous loss function and second moments of the data.
Hence,
the theorem confirms the empirical observations of the fact  that Lipschitz-continuous alternatives to the least-squares loss can yield effective learning under very weak assumptions on the data.

The proof of Theorem~\ref{robustbound} is based on the risk bound in Corollary~\ref{riskcor} and on Lipschitz and Rademacher properties of neural networks \citep{Golowich17,Taheri20}---see Section~\ref{sec:robustproof}.

Theorem~\ref{robustbound} is the first statistical guarantee for deep learning with unbounded, 
Lipschitz-continuous loss functions.
Yet,
the rates depend very similarly on the dimensions of the data and the network as the known rates for deep learning with bounded or least-squares loss  \citep{Anthony09,Golowich17,Hieber2017,Neyshabur2015,Taheri20}:
the rates in Theorem~\ref{robustbound} have basically a $1/\sqrt{\nbrsamples}$ dependence on the number of samples,
no explicit dependence on the network's input dimension and width,
and an exponential dependence on the network's depth if  $\radius>1$ and at most a linear dependence on the depth otherwise.
Hence, our results support the use of robust loss functions, such as Huber loss, not only for heavily corrupted data.

But still,
the most interesting case for robust loss functions is unbounded and non-Gaussian  data.
The specifics of the data are encapsulated in the quantities~\constinput\ and \constnoise;
broadly speaking,
Theorem~\ref{robustbound} ensures that the empirical-risk minimizer estimates the parameters effectively as long as the second moments of the input data and of the noise are reasonably small.
This assumption is, of course, much weaker than the usual  assumption of bounded or sub-Gaussian input data and noise \citep{Hieber2017,Taheri20}.
The following example illustrates a generic case where the weaker assumptions are crucial.
\begin{example}[Flawed input data]
A generic example where robust methods are useful is when parts of the input data are flawed. 
Flaws can be constructed in an adversarial manner, such as described in~\citet{Moosavi17}, for example,
or they can stem from a nonadversarial source, such as a result of measurement errors.
To fix ideas, 
assume that the components $\inputvE_1,\dots,\inputvE_{\nbrinput}$ of the input  are i.i.d.~and each sampled from a distribution~$\corruptionprob$ with probability~$\corruptionlevel$ and from a centered normal distribution  with variance~$\sigma^2$  otherwise.
We can think of $\corruptionprob$ as the type of corruption and $\corruptionlevel\in[0,1]$ as the level of corruption in the data.

Since we want to focus on the input data,
we just assume that the second moment of the noise $\outputE-\functionTF{\inputv}$ is bounded (for example, $\outputE-\functionTF{\inputv}\sim \mathcal N_1[0,1]$).

Consider first $\corruptionlevel=1$, that is, none of the input vectors are corrupted.
Then, 
$\constinput=\sqrt{\nbrinput}\sigma$,
and  Theorem~\ref{robustbound} yields the rate $\sigma\radiusb^{\nbrlayers+1}(\nbrlayers+1)\sqrt{\nbrinput/\nbrsamples}$.
This rate is virtually the same as the one that follows from  combining the bound for the Rademacher complexity in \citet[Theorem~1]{Golowich17} and the risk bound in \citet[Theorem~8]{Bartlett02},
but in contrast to those results,
Theorem~\ref{robustbound} holds for unbounded loss functions.
In any case, the agreement illustrates that Theorem~\ref{robustbound} yields good rates in the special case of few or no corrupted inputs.
More broadly speaking,
the agreement highlights the fact that Theorem~\ref{robustbound} is not only useful for corrupted data but for learning with unbounded, 
Lipschitz-continuous loss functions, such as in regression-type settings, more generally.

Consider now $\corruptionlevel=1$, that is, about~$\corruptionlevel\nbrsamples$ of the total~$\nbrsamples$ input vectors are corrupted.
One can check readily that $\constinput=(1-\corruptionlevel)\sqrt{\nbrinput}\sigma+\corruptionlevel \sqrt{\nbrinput E_{\corruptionprob}[(\corruptionvar)^2]}$,
where $\corruptionvar\sim\corruptionprob$.
Consequently,
as long as $\corruptionlevel E_{\corruptionprob}[(\corruptionvar)^2]\lesssim \sigma^2$,
Theorem~\ref{robustbound}
yields the same rate for the corrupted case as for the uncorrupted case.
As a concrete example,
let $\corruptionprob$ be a log-normal distribution (a standard example of a heavy-tailed distribution) with parameters $(0,\gamma^2)$.
Then,
$\constinput=(1-\corruptionlevel)\sqrt{\nbrinput}\sigma+\corruptionlevel \sqrt{\nbrinput}e^{\gamma^2}$.
Hence, as long as~$\sigma$ and~$\gamma$ are reasonably small,
Theorem~\ref{robustbound} ensures effective learning whatever the fraction of corrupted data is.
More generally,
these findings illustrate the usefulness of 
Theorem~\ref{robustbound} for deep learning with corrupted input data.
\end{example}

We have restricted ourselves to the popular ReLU activation functions,
but the robustness properties of Huber loss, absolute deviation, and so forth, are not tied to this type of activation.
For example,
our proof extends directly to all Lipschitz-continuous activation functions that satisfy $\activation^j[\zero_{\nbrparameterj}]=\zero_{\nbrparameterj}$ (such as leaky ReLU, for example).
Relaxing the assumptions on the activation functions further would require generalizing the results of \citet{Golowich17} and \citet{Taheri20} that we use in our proofs,
but importantly,  
the risk bounds stated in Section~\ref{sec:riskbound} do not impose any restrictions on the functions $\function\in\functionS$ and, therefore, do not limit our choice of the activation functions.

\section{Proofs}\label{sec:proofs}
In this section, we establish very detailed proofs.

\subsection{Proof of Theorem~\ref{riskbound}}\label{sec:proofriskbound}
We first give a proof for the risk bound established in Section~\ref{sec:riskbound}.

\begin{proof}[of~Theorem~\ref{riskbound}]
  The key idea is to direct the problem towards an empirical process whose expectation is proportional to the Rademacher complexity, 
and whose deviation from the expectation is controlled by a concentration inequality.

Before we start,
we introduce the shorthand
\begin{equation*}
\emppro\deq  \sup_{\functionG\in\functionS}\absBBBB{\frac{1}{\nbrsamples}\sum_{i=1}^{\nbrsamples}\biggl(\lossFB{\outputi-\functionGF{\inputvi}}-E_{(\outputE,\inputv)}\Bigl[\lossFB{\outputE-\functionGF{\inputv}}\Bigr]\biggr)}\,.
\end{equation*}
The quantity~$\emppro$ is the above-mentioned empirical process.

\emph{Step~1:} 
We first show that
\begin{equation*}
  E_{(\outputE,\inputv)}\Bigl[\loss\bigl[\outputE-\functionF{\inputv}\bigr]\Bigr]\leq \frac{1}{\nbrsamples}\sum_{i=1}^{\nbrsamples}\lossFB{\outputi-\functionF{\inputvi}}+2E_{(\outputE_1,\inputv_1),\dots,(\outputE_{\nbrsamples},\inputv_{\nbrsamples})}[\emppro]+\emppro-2E_{(\outputE_1,\inputv_1),\dots,(\outputE_{\nbrsamples},\inputv_{\nbrsamples})}[\emppro]\,.
\end{equation*}
After this first step, it remains to control the expectation of the empirical process~\emppro\ (Step~2) and the deviation of the empirical process from its expectation (Steps~3 and~4).

The proof of the first step is based on elementary algebra.
We 1.~add a zero-valued term,
2.~use the linearity of finite sums,
3.~use the fact that $a-b\leq a+\abs{b}$,
4.~take the supremum over~$\functionS$ in the second term,
5.~invoke the definition of~\emppro,
and 6.~add a zero-valued term to find
\begingroup
\allowdisplaybreaks
\begin{align*}
 & E_{(\outputE,\inputv)}\Bigl[\loss\bigl[\outputE-\functionF{\inputv}\bigr]\Bigr]\\
 &= \frac{1}{\nbrsamples}\sum_{i=1}^{\nbrsamples}\lossFB{\outputi-\functionF{\inputvi}}-\biggl(\frac{1}{\nbrsamples}\sum_{i=1}^{\nbrsamples}\lossFB{\outputi-\functionF{\inputvi}}-E_{(\outputE,\inputv)}\Bigl[\lossFB{\outputE-\functionF{\inputv}}\Bigr]\biggr)\\
 &=\frac{1}{\nbrsamples}\sum_{i=1}^{\nbrsamples}\lossFB{\outputi-\functionF{\inputvi}}-\frac{1}{\nbrsamples}\sum_{i=1}^{\nbrsamples}\biggl(\lossFB{\outputi-\functionF{\inputvi}}-E_{(\outputE,\inputv)}\Bigl[\lossFB{\outputE-\functionF{\inputv}}\Bigr]\biggr)\\
 &\leq \frac{1}{\nbrsamples}\sum_{i=1}^{\nbrsamples}\lossFB{\outputi-\functionF{\inputvi}}+\absBBBB{\frac{1}{\nbrsamples}\sum_{i=1}^{\nbrsamples}\biggl(\lossFB{\outputi-\functionF{\inputvi}}-E_{(\outputE,\inputv)}\Bigl[\lossFB{\outputE-\functionF{\inputv}}\Bigr]\biggr)}\\
 &\leq \frac{1}{\nbrsamples}\sum_{i=1}^{\nbrsamples}\lossFB{\outputi-\functionF{\inputvi}}+\sup_{\functionG\in\functionS}\absBBBB{\frac{1}{\nbrsamples}\sum_{i=1}^{\nbrsamples}\biggl(\lossFB{\outputi-\functionGF{\inputvi}}-E_{(\outputE,\inputv)}\Bigl[\lossFB{\outputE-\functionGF{\inputv}}\Bigr]\biggr)}\\
 &= \frac{1}{\nbrsamples}\sum_{i=1}^{\nbrsamples}\lossFB{\outputi-\functionF{\inputvi}}+\emppro\\
 &= \frac{1}{\nbrsamples}\sum_{i=1}^{\nbrsamples}\lossFB{\outputi-\functionF{\inputvi}}+2E_{(\outputE_1,\inputv_1),\dots,(\outputE_{\nbrsamples},\inputv_{\nbrsamples})}[\emppro]+\emppro-2E_{(\outputE_1,\inputv_1),\dots,(\outputE_{\nbrsamples},\inputv_{\nbrsamples})}[\emppro]\,,
\end{align*}
\endgroup
as desired.

\emph{Step~2:}
We now show that
\begin{equation*}
  E_{(\outputE,\inputv)}\Bigl[\loss\bigl[\outputE-\functionF{\inputv}\bigr]\Bigr]\leq \frac{1}{\nbrsamples}\sum_{i=1}^{\nbrsamples}\lossFB{\outputi-\functionF{\inputvi}}+16\constlip\radecomplex+ \frac{8\constlip\constnoise}{\sqrt{\nbrsamples}}+\emppro-2E_{(\outputE_1,\inputv_1),\dots,(\outputE_{\nbrsamples},\inputv_{\nbrsamples})}[\emppro]\,.
\end{equation*}
This step takes care of one of the $2E_{(\outputE_1,\inputv_1),\dots,(\outputE_{\nbrsamples},\inputv_{\nbrsamples})}[\emppro]$ in the previous bound.

The key ingredients are symmetrization and contraction arguments, 
the Lipschitz property of the loss function, 
and the concentration of sums of Rademacher random variables.
We introduce $(\outputE_1',\inputv_1'),\dots,(\outputE_{\nbrsamples}',\inputv_{\nbrsamples}')\in\R\times\R^{\nbrinput}$ as random variables that are i.i.d.~copies of $(\outputE,\inputv)$ and independent of the rest of the data.
We first render the empirical process ``symmetric.''
We use 
1.~the definition of the empirical process~\emppro,
2.~the i.i.d.~assumption on the data,
3.~the linearity of integrals and finite sums,
4.~dominated convergence,
5.~the i.i.d.~assumption on the data and the properties of the Rademacher random variables,
6.~the linearity of finite sums, the triangle inequality, and the properties of suprema,
and 7.~the linearity of integrals and the i.i.d.~assumption on the data
to find that
\begin{align*}
  &2E_{(\outputE_1,\inputv_1),\dots,(\outputE_{\nbrsamples},\inputv_{\nbrsamples})}[\emppro]\\
&=2E_{(\outputE_1,\inputv_1),\dots,(\outputE_{\nbrsamples},\inputv_{\nbrsamples})}\Biggl[\sup_{\functionG\in\functionS}\absBBBB{\frac{1}{\nbrsamples}\sum_{i=1}^{\nbrsamples}\biggl(\lossFB{\outputi-\functionGF{\inputvi}}-E_{(\outputE,\inputv)}\Bigl[\lossFB{\outputE-\functionGF{\inputv}}\Bigr]\biggr)}\Biggr]\\
&=2E_{(\outputE_1,\inputv_1),\dots,(\outputE_{\nbrsamples},\inputv_{\nbrsamples})}\Biggl[\sup_{\functionG\in\functionS}\absBBBB{\frac{1}{\nbrsamples}\sum_{i=1}^{\nbrsamples}\biggl(\lossFB{\outputi-\functionGF{\inputvi}}-E_{(\outputE_1',\inputv_1'),\dots,(\outputE_{\nbrsamples}',\inputv_{\nbrsamples}')}\Bigl[\lossFB{\outputE_{i}'-\functionGF{\inputv_{i}'}}\Bigr]\biggr)}\Biggr]\\
&=2E_{(\outputE_1,\inputv_1),\dots,(\outputE_{\nbrsamples},\inputv_{\nbrsamples})}\Biggl[\sup_{\functionG\in\functionS}\absBBBB{E_{(\outputE_1',\inputv_1'),\dots,(\outputE_{\nbrsamples}',\inputv_{\nbrsamples}')}\Biggl[\frac{1}{\nbrsamples}\sum_{i=1}^{\nbrsamples}\Bigl(\lossFB{\outputi-\functionGF{\inputvi}}-\lossFB{\outputE_{i}'-\functionGF{\inputv_{i}'}}\Bigr)\Biggr]}\Biggr]\\
&\leq 2E_{(\outputE_1,\inputv_1),\dots,(\outputE_{\nbrsamples},\inputv_{\nbrsamples}),(\outputE_1',\inputv_1'),\dots,(\outputE_{\nbrsamples}',\inputv_{\nbrsamples}')}\Biggl[\sup_{\functionG\in\functionS}\absBBB{\frac{1}{\nbrsamples}\sum_{i=1}^{\nbrsamples}\Bigl(\lossFB{\outputi-\functionGF{\inputvi}}-\lossFB{\outputE_{i}'-\functionGF{\inputv_{i}'}}\Bigr)}\Biggr]\\
  &= 2E_{(\outputE_1,\inputv_1),\dots,(\outputE_{\nbrsamples},\inputv_{\nbrsamples}),(\outputE_1',\inputv_1'),\dots,(\outputE_{\nbrsamples}',\inputv_{\nbrsamples}'),\radrv_1,\dots,\radrv_{\nbrsamples}}\Biggl[\sup_{\functionG\in\functionS}\absBBB{\frac{1}{\nbrsamples}\sum_{i=1}^{\nbrsamples}\radrvi\Bigl(\lossFB{\outputi-\functionGF{\inputvi}}-\lossFB{\outputE_i'-\functionGF{\inputv_i'}}\Bigr)}\Biggr]\\
  &\leq 2E_{(\outputE_1,\inputv_1),\dots,(\outputE_{\nbrsamples},\inputv_{\nbrsamples}),(\outputE_1',\inputv_1'),\dots,(\outputE_{\nbrsamples}',\inputv_{\nbrsamples}'),\radrv_1,\dots,\radrv_{\nbrsamples}}\Biggl[\sup_{\functionG\in\functionS}\absBBB{\frac{1}{\nbrsamples}\sum_{i=1}^{\nbrsamples}\radrvi\lossFB{\outputi-\functionGF{\inputvi}}}+\sup_{\functionG\in\functionS}\absBBB{\frac{1}{\nbrsamples}\sum_{i=1}^{\nbrsamples}\radrvi\lossFB{\outputE_i'-\functionGF{\inputv_i'}}}\Biggr]\\
  &\leq 4E_{(\outputE_1,\inputv_1),\dots,(\outputE_{\nbrsamples},\inputv_{\nbrsamples}),\radrv_1,\dots,\radrv_{\nbrsamples}}\Biggl[\sup_{\functionG\in\functionS}\absBBB{\frac{1}{\nbrsamples}\sum_{i=1}^{\nbrsamples}\radrvi\lossFB{\outputi-\functionGF{\inputvi}}}\Biggr]\,.
\end{align*}

We then apply a contraction argument.
We use
1.~the contraction principle in~\cite[second part of Theorem~11.6 on pp.~324--325]{Boucheron13} with $x_{i,\functionG}\deq\outputE_i-\functionGF{\inputvi}$ (with some abuse of notation), $\varphi_i\deq \loss$ (see our assumptions for~\loss\ on Page~\pageref{Lipschitz}), and $\Psi$ the identity function,
2.~the insertion of zero-valued term,
3.~the linearity of finite sums, the triangle inequality, and the properties of suprema,
4.~the linearity of integrals,
and 5.~Definition~\ref{complexitymeasures} of the Rademacher complexity~\radecomplex\ to show that
\begingroup
\allowdisplaybreaks
\begin{align*}
  &2E_{(\outputE_1,\inputv_1),\dots,(\outputE_{\nbrsamples},\inputv_{\nbrsamples})}[\emppro]\\
  &\leq 8\constlip E_{(\outputE_1,\inputv_1),\dots,(\outputE_{\nbrsamples},\inputv_{\nbrsamples}),\radrv_1,\dots,\radrv_{\nbrsamples}}\Biggl[\sup_{\functionG\in\functionS}\absBBB{\frac{1}{\nbrsamples}\sum_{i=1}^{\nbrsamples}\radrvi\bigl(\outputi-\functionGF{\inputvi}\bigr)}\Biggr]\\
  &= 8\constlip E_{(\outputE_1,\inputv_1),\dots,(\outputE_{\nbrsamples},\inputv_{\nbrsamples}),\radrv_1,\dots,\radrv_{\nbrsamples}}\Biggl[\sup_{\functionG\in\functionS}\absBBB{\frac{1}{\nbrsamples}\sum_{i=1}^{\nbrsamples}\radrvi\bigl(\outputi-\functionTF{\inputv_i}+\functionTF{\inputv_i}-\functionGF{\inputvi}\bigr)}\Biggr]\\
  &\leq 8\constlip E_{(\outputE_1,\inputv_1),\dots,(\outputE_{\nbrsamples},\inputv_{\nbrsamples}),\radrv_1,\dots,\radrv_{\nbrsamples}}\Biggl[\sup_{\functionG\in\functionS}\absBBB{\frac{1}{\nbrsamples}\sum_{i=1}^{\nbrsamples}\radrvi\functionGF{\inputvi}}+\absBBB{\frac{1}{\nbrsamples}\sum_{i=1}^{\nbrsamples}\radrvi\functionTF{\inputvi}}+\absBBB{\frac{1}{\nbrsamples}\sum_{i=1}^{\nbrsamples}\radrvi\bigl(\outputi-\functionTF{\inputv_i}\bigr)}\Biggr]\\
  &\leq 16\constlip E_{(\outputE_1,\inputv_1),\dots,(\outputE_{\nbrsamples},\inputv_{\nbrsamples}),\radrv_1,\dots,\radrv_{\nbrsamples}}\Biggl[\sup_{\functionG\in\functionS}\absBBB{\frac{1}{\nbrsamples}\sum_{i=1}^{\nbrsamples}\radrvi\functionGF{\inputvi}}\Biggr]\\
&~~~~~~~~~~~~~~~~~~~~~~~~~~~~~~~~~~~~~~~~~~~~~+\frac{8\constlip}{\nbrsamples}E_{(\outputE_1,\inputv_1),\dots,(\outputE_{\nbrsamples},\inputv_{\nbrsamples}),\radrv_1,\dots,\radrv_{\nbrsamples}}\Biggl[\absBBB{\sum_{i=1}^{\nbrsamples}\radrvi\bigl(\outputi-\functionTF{\inputv_i}\bigr)}\Biggr]\\
  &=16\constlip \radecomplex+\frac{8\constlip}{\nbrsamples}E_{(\outputE_1,\inputv_1),\dots,(\outputE_{\nbrsamples},\inputv_{\nbrsamples}),\radrv_1,\dots,\radrv_{\nbrsamples}}\Biggl[\absBBB{\sum_{i=1}^{\nbrsamples}\radrvi\bigl(\outputi-\functionTF{\inputv_i}\bigr)}\Biggr]\,.
\end{align*}
\endgroup

We then use a contraction property of Rademacher random variables to control the second term.
We use
1.~the law of iterated expectations~\cite[Display~(5.1.5) on p.~228]{Durrett10},
2.~Khinchin's inequality~\cite[p.~232]{Haagerup81},
3.~again the law of iterated expectations,
4.~Jensen's inequality~\cite[Theorem~1.5.1 on p.~23]{Durrett10},
5.~the linearity of integrals and the i.i.d.~assumption on the data,
and 6.~the definition of~\constnoise\
to derive that
\begingroup
\allowdisplaybreaks
\begin{align*}
  &E_{(\outputE_1,\inputv_1),\dots,(\outputE_{\nbrsamples},\inputv_{\nbrsamples}),\radrv_1,\dots,\radrv_{\nbrsamples}}\Biggl[\absBBB{\sum_{i=1}^{\nbrsamples}\radrvi\bigl(\outputi-\functionTF{\inputv_i}\bigr)}\Biggr]\\
&=E_{(\outputE_1,\inputv_1),\dots,(\outputE_{\nbrsamples},\inputv_{\nbrsamples}),\radrv_1,\dots,\radrv_{\nbrsamples}}\Biggl[E_{(\outputE_1,\inputv_1),\dots,(\outputE_{\nbrsamples},\inputv_{\nbrsamples}),\radrv_1,\dots,\radrv_{\nbrsamples}}\Biggl[\absBBB{\sum_{i=1}^{\nbrsamples}\radrvi\bigl(\outputi-\functionTF{\inputv_i}\bigr)}\mid(\outputE_1,\inputv_1),\dots,(\outputE_{\nbrsamples},\inputv_{\nbrsamples})\Biggr]\Biggr]\\
&\leq E_{(\outputE_1,\inputv_1),\dots,(\outputE_{\nbrsamples},\inputv_{\nbrsamples}),\radrv_1,\dots,\radrv_{\nbrsamples}}\Biggl[E_{(\outputE_1,\inputv_1),\dots,(\outputE_{\nbrsamples},\inputv_{\nbrsamples}),\radrv_1,\dots,\radrv_{\nbrsamples}}\biggl(\sum_{i=1}^{\nbrsamples}\bigl(\outputi-\functionTF{\inputv_i}\bigr)^2\biggr)^{1/2}\mid(\outputE_1,\inputv_1),\dots,(\outputE_{\nbrsamples},\inputv_{\nbrsamples})\Biggr]\\
&= E_{(\outputE_1,\inputv_1),\dots,(\outputE_{\nbrsamples},\inputv_{\nbrsamples})}\Biggl[\biggl(\sum_{i=1}^{\nbrsamples}\bigl(\outputi-\functionTF{\inputv_i}\bigr)^2\biggr)^{1/2}\Biggr]\\
&\leq\Biggl( E_{(\outputE_1,\inputv_1),\dots,(\outputE_{\nbrsamples},\inputv_{\nbrsamples})}\biggl[\sum_{i=1}^{\nbrsamples}\bigl(\outputi-\functionTF{\inputv_i}\bigr)^2\biggr]\Biggr)^{1/2}\\
&=\biggl(\nbrsamples E_{(\outputE,\inputv)}\Bigl[\bigl(\outputE-\functionTF{\inputv}\bigr)^2\Bigr]\biggr)^{1/2}\\
&=\sqrt{\nbrsamples}\constnoise\,.
\end{align*}
\endgroup

Combining the inequalities derived in this step yields
\begin{equation*}
   2E_{(\outputE_1,\inputv_1),\dots,(\outputE_{\nbrsamples},\inputv_{\nbrsamples})}[\emppro]\leq 16\constlip\radecomplex+ \frac{8\constlip\constnoise}{\sqrt{\nbrsamples}}\,,
\end{equation*}
and combining this result with the result of Step~1 then finally gives the desired statement.

\emph{Step~3:}
We now show that 
\begin{equation*}
  E_{(\outputE_i,\inputv_i)}\Biggl[\sup_{\functionG\in\functionS}\biggl(\lossFB{\outputE_i-\functionGF{\inputv_i}}-E_{(\outputE,\inputv)}\Bigl[\lossFB{\outputE-\functionGF{\inputv}}\Bigr]\biggr)^2\Biggr]\leq (3\constlip\constenv+3\constlip\constnoise)^2
\end{equation*}
for all $i\in\{1,\dots,\nbrsamples\}$.
This bound will be essential for applying a concentration inequality in the following step.

We use elementary tools to connect the left-hand side with the complexity measures in Definition~\ref{complexitymeasures}.
Specifically,
we
1.~invoke the i.i.d.~assumption for the data and the linearity of integrals,
2.~use the fact that $a\leq \abs{a}$,
3.~invoke the Lipschitz condition~\eqref{Lipschitz} for  the loss~\loss,
4.~add a zero-valued term,
5.~use the triangle inequality and the linearity of integrals,
6.~apply dominated convergence and Jensen's inequality,
7.~use $(a+b+c+d)^2\leq 4(a^2+b^2+c^2+d^2)$ according to Lemma~\ref{binomial} in Section~\ref{sec:auxresult}, the properties of suprema, and the linearity of integrals,
8.~use the linearity of integrals and the i.i.d.~assumption on the data,
9.~invoke Definition~\ref{complexitymeasures} of~\constenv\ and~\constnoise,
and finally 10.~$a^2+b^2\leq (a+b)^2$ for nonnegative $a,b$ to find
\begingroup
\allowdisplaybreaks
\begin{align*}
  &E_{(\outputE_i,\inputv_i)}\Biggl[\sup_{\functionG\in\functionS}\biggl(\lossFB{\outputE_i-\functionGF{\inputv_i}}-E_{(\outputE,\inputv)}\Bigl[\lossFB{\outputE-\functionGF{\inputv}}\Bigr]\biggr)^2\Biggr]\\
  &=E_{(\outputE_i,\inputv_i)}\Biggl[\sup_{\functionG\in\functionS}\biggl(E_{(\outputE,\inputv)}\Bigl[\lossFB{\outputE_i-\functionGF{\inputv_i}}-\lossFB{\outputE-\functionGF{\inputv}}\Bigr]\biggr)^2\Biggr]\\
  &\leq E_{(\outputE_i,\inputv_i)}\Biggl[\sup_{\functionG\in\functionS}\biggl(E_{(\outputE,\inputv)}\Bigl[\absB{\lossFB{\outputE_i-\functionGF{\inputv_i}}-\lossFB{\outputE-\functionGF{\inputv}}}\Bigr]\biggr)^2\Biggr]\\
  &\leq E_{(\outputE_i,\inputv_i)}\Biggl[\sup_{\functionG\in\functionS}\biggl(E_{(\outputE,\inputv)}\Bigl[\constlip\absB{\outputE_i-\functionGF{\inputv_i}-\outputE+\functionGF{\inputv}}\Bigr]\biggr)^2\Biggr]\\
  &= E_{(\outputE_i,\inputv_i)}\Biggl[\sup_{\functionG\in\functionS}\biggl(E_{(\outputE,\inputv)}\Bigl[\constlip\absB{\outputE_i-\functionGF{\inputv_i}+\functionTF{\inputvi}-\functionTF{\inputv_i}-\functionTF{\inputv}+\functionTF{\inputv}-\outputE+\functionGF{\inputv}}\Bigr]\biggr)^2\Biggr]\\
  &\leq (\constlip)^2 E_{(\outputE_i,\inputv_i)}\Biggl[\sup_{\functionG\in\functionS}\biggl(E_{(\outputE,\inputv)}\Bigl[\absB{\functionGF{\inputv_i}-\functionTF{\inputv_i}}+\absB{\functionGF{\inputv}-\functionTF{\inputv}}+\absB{\outputE_i-\functionTF{\inputvi}}+\absB{\outputE-\functionTF{\inputv}}\Bigr]\biggr)^2\Biggr]\\
  &\leq (\constlip)^2 E_{(\outputE,\inputv),(\outputE_i,\inputv_i)}\biggl[\sup_{\functionG\in\functionS}\Bigl(\absB{\functionGF{\inputv_i}-\functionTF{\inputv_i}}+\absB{\functionGF{\inputv}-\functionTF{\inputv}}+\absB{\outputE_i-\functionTF{\inputvi}}+\absB{\outputE-\functionTF{\inputv}}\Bigr)^2\biggr]\\
  &\leq 4(\constlip)^2 E_{(\outputE,\inputv),(\outputE_i,\inputv_i)}\biggl[\sup_{\functionG\in\functionS}\absB{\functionGF{\inputv_i}-\functionTF{\inputv_i}}^2+\sup_{\functionG\in\functionS}\absB{\functionGF{\inputv}-\functionTF{\inputv}}^2+\absB{\outputE_i-\functionTF{\inputvi}}^2+\absB{\outputE-\functionTF{\inputv}}^2\biggr]\\
  &=8(\constlip)^2 E_{(\outputE,\inputv)}\biggl[\sup_{\functionG\in\functionS}\absB{\functionGF{\inputv}-\functionTF{\inputv}}^2\biggr]+8(\constlip)^2E_{(\outputE,\inputv)}\Bigl[\absB{\outputE-\functionTF{\inputv}}^2\Bigr]\\
&= 8(\constlip)^2(\constenv)^2+8(\constlip)^2(\constnoise)^2\\
&\leq (3\constlip\constenv+3\constlip\constnoise\bigr)^2\,,
\end{align*}
\endgroup
as desired.

\emph{Step~4:} We now show that 
\begin{equation*}
  P_{(\outputE_1,\inputv_1),\dots,(\outputE_{\nbrsamples},\inputv_{\nbrsamples})}\biggl\{\emppro-2E_{(\outputE_1,\inputv_1),\dots,(\outputE_{\nbrsamples},\inputv_{\nbrsamples})}[\emppro]\geq \frac{228\constlip\constenv+228\constlip\constnoise}{\sqrt{n}\level}\biggr\}\leq \level\,.
\end{equation*}
This deviation inequality controls the remaining term in our bound.

The proof is based on Step~3 and  a concentration result by~\citet{Lederer14}.
The coordinates of the random vectors in \citet[Section~2]{Lederer14} are in our case (with some abuse of notation) 
$Z_i[\functionG]\deq \lossF{\outputE_i-\functionGF{\inputv_i}}-E_{(\outputE,\inputv)}[\lossF{\outputE-\functionGF{\inputv}}]$.
As coordinates of the envelope,
we simply take $\mathcal{E}_i\deq \sup_{\functionG\in\functionS}\abs{\lossF{\outputE_i-\functionGF{\inputv_i}}-E_{(\outputE,\inputv)}[\lossF{\outputE-\functionGF{\inputv}}]}$.
According to Step~3,
it holds that $\sigma\leq M\leq 3\constlip\constenv+3\constlip\constnoise$ for $p=2$---see their Equation~(4).
Hence, \citet[Corollary~3]{Lederer14} yields (with $\epsilon\deq 1$, $l\deq 1$, and $p\deq2$)
 \begin{multline*}
  P_{(\outputE_1,\inputv_1),\dots,(\outputE_{\nbrsamples},\inputv_{\nbrsamples})}\bigl\{\emppro-2E_{(\outputE_1,\inputv_1),\dots,(\outputE_{\nbrsamples},\inputv_{\nbrsamples})}[\emppro]\geq v\bigr\}\leq \frac{72 (3\constlip\constenv+3\constlip\constnoise)}{\sqrt{\nbrsamples} v}+\frac{4(3\constlip\constenv+3\constlip\constnoise)}{\sqrt{\nbrsamples}v}\\
=\frac{228\constlip\constenv+228\constlip\constnoise}{\sqrt{\nbrsamples} v}
\end{multline*}
for all $v\in(0,\infty)$.
Setting $v\deq (228\constlip\constenv+228\constlip\constnoise)/(\sqrt{\nbrsamples}t)$ then gives the desired result.

Combining Steps~2 and~4 and using that $\level<1$ finally yields the bound stated in the theorem.
\end{proof}

\subsection{An Auxilliary Result}\label{sec:auxresult}
We now state an simple auxilliary result that was used in the above-stated proof of Theorem~\ref{riskbound}. 
The result is very standard,
but for the sake of completeness,
we prove it nevertheless.
\begin{lemma}[Binomial]\label{binomial}
For every $a,b,c,d\in\R$,
it holds that
\begin{equation*}
  (a+b+c+d)^2\leq 4a^2+4b^2+4c^2+4d^2\,.
\end{equation*}
\end{lemma}

\begin{proof}[of Lemma~\ref{binomial}]
Noting that
  \begin{align*}
    2uv=-(u-v)^2+u^2+v^2\leq u^2+v^2
  \end{align*}
for all $u,v\in\R$,
we find 
  \begin{align*}
    &(a+b+c+d)^2\\
&=a^2+b^2+c^2+d^2+2ab+2ac+2ad+2bc+2bd+2cd\\
&\leq a^2+b^2+c^2+d^2+a^2+b^2+a^2+c^2+a^2+d^2+b^2+c^2+b^2+d^2+c^2+d^2\\
&=4a^2+4b^2+4c^2+4d^2\,,
  \end{align*}
  as desired.
\end{proof}

\subsection{Proof of Theorem~\ref{robustbound}}
\label{sec:robustproof}

We finally give a proof for the robust guarantee established in Section~\ref{sec:deeplearning}.

\begin{proof}[of Theorem~\ref{robustbound}]
We need to control the terms of the right-hand side of the inequality in Corollary~\ref{riskcor}.
Two results that are especially important in our derivations are a Lipschitz property of neural networks developed in \citet{Taheri20} and a bound for the Rademacher complexity of neural networks developed in \citet{Golowich17}.

\emph{Step~1:} We first show that with probability at least $1-2\level$,
it holds  that
  \begin{equation*}
    E_{(\outputE,\inputv)}\Bigl[\loss\bigl[\outputE-\ermF{\inputv}\bigr]\Bigr]\leq E_{(\outputE,\inputv)}\Bigl[\loss\bigl[\outputE-\functionTF{\inputv}\bigr]\Bigr]+16\constlip\radecomplex+237\constlip\frac{\constenv+\constnoise}{\sqrt{\nbrsamples}\level}\,.
  \end{equation*}
This first step takes care of the empirical loss in the bound of Corollary~\ref{riskcor}.

The proof of the first step is based on Corollary~\ref{riskcor} and Markov's inequality.
We use 
1.~the definition of $\erm$ as a risk minimizer in~\eqref{erm},
2.~a rearrangement of the terms and the linearity of finite sums,
3.~Markov's inequality \cite[Display~(1.6.1) on p.~29]{Durrett10},
4.~the i.i.d.~assumption on the data and the linearity of integrals,
5.~a consolidation of the factors,
6.~the fact that $E[(v-E[v])^2]\leq E[v^2]$ and the i.i.d.~assumption on the data,
7.~the assumption $\lossF{0}=0$ on Page~\pageref{Lipschitz},
8.~the Lipschitz assumption~\eqref{Lipschitz} on the loss~\loss,
9.~the linearity of integrals and a consolidation,
and 10.~Definition~\ref{complexitymeasures} of~\constnoise\ and the fact that $\level\in(0,1)$ to find
\begingroup
\allowdisplaybreaks
\begin{align*}
  &P_{(\outputE_1,\inputv_1),\dots,(\outputE_{\nbrsamples},\inputv_{\nbrsamples})}\biggl\{  \frac{1}{\nbrsamples}\sum_{i=1}^{\nbrsamples}\lossFB{\outputE_i-\ermF{\inputvi}}\geq E_{(\outputE,\inputv)}\Bigl[\loss\bigl[\outputE-\functionTF{\inputv}\bigr]\Bigr] + \frac{\constlip\constnoise}{\sqrt{\nbrsamples}\level}\biggr\}\\
  &\leq P_{(\outputE_1,\inputv_1),\dots,(\outputE_{\nbrsamples},\inputv_{\nbrsamples})}\biggl\{  \frac{1}{\nbrsamples}\sum_{i=1}^{\nbrsamples}\lossFB{\outputE_i-\functionTF{\inputvi}}\geq E_{(\outputE,\inputv)}\Bigl[\loss\bigl[\outputE-\functionTF{\inputv}\bigr]\Bigr] + \frac{\constlip\constnoise}{\sqrt{\nbrsamples}\level}\biggr\}\\
  &=P_{(\outputE_1,\inputv_1),\dots,(\outputE_{\nbrsamples},\inputv_{\nbrsamples})}\biggl\{  \frac{1}{\nbrsamples}\sum_{i=1}^{\nbrsamples}\biggl(\lossFB{\outputE_i-\functionTF{\inputvi}}-E_{(\outputE,\inputv)}\Bigl[\loss\bigl[\outputE-\functionTF{\inputv}\bigr]\Bigr]\biggr)\geq  \frac{\constlip\constnoise}{\sqrt{\nbrsamples}\level}\biggr\}\\
&\leq \frac{E_{(\outputE_1,\inputv_1),\dots,(\outputE_{\nbrsamples},\inputv_{\nbrsamples})}\biggl[  \absBB{\sum_{i=1}^{\nbrsamples}\lossFB{\outputE_i-\functionTF{\inputvi}}-E_{(\outputE,\inputv)}\Bigl[\loss\bigl[\outputE-\functionTF{\inputv}\bigr]\Bigr]}^2 /\nbrsamples^2\biggr]}{\bigl(\fracNN{\constlip\constnoise}{\sqrt{\nbrsamples}\level}\bigr)^2}\\
  &= \frac{E_{(\outputE_1,\inputv_1)}\biggl [  \absBB{\lossFB{\outputE_1-\functionTF{\inputv_1}}-E_{(\outputE,\inputv)}\Bigl[\loss\bigl[\outputE-\functionTF{\inputv}\bigr]\Bigr]}^2 \biggr]/\nbrsamples}{\bigl(\fracNN{\constlip\constnoise}{\sqrt{\nbrsamples}\level}\bigr)^2}\\
&= \frac{E_{(\outputE_1,\inputv_1)}\biggl [  \absBB{\lossFB{\outputE_1-\functionTF{\inputv_1}}-E_{(\outputE,\inputv)}\Bigl[\loss\bigl[\outputE-\functionTF{\inputv}\bigr]\Bigr]}^2 \biggr]}{(\constlip)^2(\constnoise)^2}\cdot\level^2\\
&\leq\frac{E_{(\outputE,\inputv)}\biggl [  \absBB{\lossFB{\outputE-\functionTF{\inputv}}}^2 \biggr]}{(\constlip)^2(\constnoise)^2}\cdot\level^2\\
&= \frac{E_{(\outputE,\inputv)}\biggl [  \absBB{\lossFB{\outputE-\functionTF{\inputv}}-\lossF{0}}^2 \biggr]}{(\constlip)^2(\constnoise)^2}\cdot\level^2\\
&\leq\frac{E_{(\outputE,\inputv)}\Bigl[  (\constlip)^2\absB{\outputE-\functionTF{\inputv}-0}^2 \Bigr]}{(\constlip)^2(\constnoise)^2}\cdot\level^2\\
&=\frac{E_{(\outputE,\inputv)}\Bigl [ \absB{\outputE-\functionTF{\inputv}}^2 \Bigr]}{(\constnoise)^2}\cdot\level^2\\
  &\leq  \level\,.
\end{align*}
\endgroup
We can then conclude by plugging this result into Corollary~\ref{riskcor}.

\emph{Step~2:} We now show that with probability at least $1-2\level$,
it holds that
  \begin{equation*}
    E_{(\outputE,\inputv)}\Bigl[\loss\bigl[\outputE-\ermF{\inputv}\bigr]\Bigr]\leq E_{(\outputE,\inputv)}\Bigl[\loss\bigl[\outputE-\functionTF{\inputv}\bigr]\Bigr]+48\radiusb^{\nbrlayers+1}\constlip\sqrt{\nbrlayers+1}\frac{\constinput}{\sqrt{\nbrsamples}}+237\constlip\frac{\constenv+\constnoise}{\sqrt{\nbrsamples}\level}\,.
  \end{equation*}
This step takes care of the Rademacher complexity.

The basis for the proof is a bound for the Rademacher complexity of neural networks from~\citet{Golowich17}.
Indeed,
we use 
1.~\citet[Theorem~1]{Golowich17}, 
2.~the linearity of integrals,
3.~Jensen's inequality,
4.~the linearity of integrals,
5.~the i.i.d.~assumption on the data,
and 6.~Definition~\ref{complexitymeasures} of~\constinput\
to find
\begingroup
\allowdisplaybreaks
\begin{align*}
  \radecomplex&\leq E_{(\outputE_1,\inputv_1),\dots,(\outputE_{\nbrsamples},\inputv_{\nbrsamples})}\Biggl[\frac{3\radiusb^{\nbrlayers+1}\sqrt{\nbrlayers+1}}{\sqrt{\nbrsamples}}\sqrt{\frac{1}{\nbrsamples}\sum_{i=1}^{\nbrsamples}\normtwos{\inputvi}}\Biggr]\\
&= \frac{3\radiusb^{\nbrlayers+1}\sqrt{\nbrlayers+1}}{\sqrt{\nbrsamples}}E_{(\outputE_1,\inputv_1),\dots,(\outputE_{\nbrsamples},\inputv_{\nbrsamples})}\Biggl[\sqrt{\frac{1}{\nbrsamples}\sum_{i=1}^{\nbrsamples}\normtwos{\inputvi}}\Biggr]\\
&\leq \frac{3\radiusb^{\nbrlayers+1}\sqrt{\nbrlayers+1}}{\sqrt{\nbrsamples}}\sqrt{E_{(\outputE_1,\inputv_1),\dots,(\outputE_{\nbrsamples},\inputv_{\nbrsamples})}\biggl[{\frac{1}{\nbrsamples}\sum_{i=1}^{\nbrsamples}\normtwos{\inputvi}}\biggr]}\\
&= \frac{3\radiusb^{\nbrlayers+1}\sqrt{\nbrlayers+1}}{\sqrt{\nbrsamples}}{\sqrt{\frac{1}{\nbrsamples}\sum_{i=1}^{\nbrsamples}E_{(\outputE_1,\inputv_1),\dots,(\outputE_{\nbrsamples},\inputv_{\nbrsamples})}\bigl[\normtwos{\inputvi}}\bigr]}\\
&= \frac{3\radiusb^{\nbrlayers+1}\sqrt{\nbrlayers+1}}{\sqrt{\nbrsamples}}\sqrt{E_{(\outputE,\inputv)}\bigl[\normtwos{\inputv}\bigr]}\\
&= \frac{3\radiusb^{\nbrlayers+1}\sqrt{\nbrlayers+1}\constinput}{\sqrt{\nbrsamples}}\,.
\end{align*}
\endgroup
We can then conclude by plugging this inequality into the result of Step~1.

\emph{Step~3:} We now show that with probability at least $1-2\level$,
it holds that
  \begin{equation*}
    E_{(\outputE,\inputv)}\Bigl[\loss\bigl[\outputE-\ermF{\inputv}\bigr]\Bigr]\leq E_{(\outputE,\inputv)}\Bigl[\loss\bigl[\outputE-\functionTF{\inputv}\bigr]\Bigr]+48\radiusb^{\nbrlayers+1}\constlip\sqrt{\nbrlayers+1}\frac{\constinput}{\sqrt{\nbrsamples}}+948\constlip\frac{\radiusb^{\nbrlayers+1}\nbrlayers\constinput+\constnoise}{\sqrt{\nbrsamples}\level}\,.
  \end{equation*}
  This step takes care of the size of the envelope.
(We do not attempt to  optimize constants anywhere in our proofs.)

The key idea here is to apply a Lipschitz property of neural networks derived in~\citet{Taheri20}.
We use
1.~Definition~\ref{complexitymeasures} of~\constenv,
2.~the specification of the set~\functionS\ on Page~\pageref{networks} and the assumption that on $\parameterS$ on Page~\pageref{setweight},
3.~\cite[Proposition~2]{Taheri20} and the definition of the set~\parameterS\ on Page~\pageref{setweight},
4.~the fact that $(u-v)^2\leq 2u^2+2v^2$,
5.~again the definition of~\parameterS,
6.~a consolidation and the linearity of integrals,
and 7.~Definition~\ref{complexitymeasures} of~\constinput\ to find
\begingroup
\allowdisplaybreaks
\begin{align*}
  (\constenv)^2
&=  {E_{(\outputE,\inputv)}\biggl[\sup_{\function\in\functionS}\absB{\functionF{\inputv}-\functionTF{\inputv}}^2\biggr]}\\
&\leq  {E_{(\outputE,\inputv)}\biggl[\sup_{\parameter,\parameterG\in\parameterS}\absB{\function_{\parameter}[\inputv]-\function_{\parameterG}[\inputv]}^2\biggr]}\\
&\leq  {E_{(\outputE,\inputv)}\Biggl[\sup_{\parameter,\parameterG\in\parameterS}\biggl\{4\radiusb^{2\nbrlayers}\nbrlayers\normtwos{\inputv}
\sum_{j=0}^{\nbrlayers}\normF{\parameterE^j-\parameterGE^j}^2\biggr\}\Biggr]}\\
&\leq  {E_{(\outputE,\inputv)}\Biggl[\sup_{\parameter,\parameterG\in\parameterS}\biggl\{4\radiusb^{2\nbrlayers}\nbrlayers\normtwos{\inputv}\biggl(2\sum_{j=0}^{\nbrlayers}\normF{\parameterE^j}^2+2\sum_{j=0}^{\nbrlayers}\normF{\parameterGE^j}^2\biggr)\biggr\}\Biggr]}\\
&\leq  {E_{(\outputE,\inputv)}\Bigl[4\radiusb^{2\nbrlayers}\nbrlayers\normtwos{\inputv}\bigl(2\nbrlayers\radiusb^{2}+2\nbrlayers\radiusb^{2}\bigr)\Bigr]}\\
&= 16\radiusb^{2\nbrlayers+2}\nbrlayers^2{E_{(\outputE,\inputv)}\bigl[\normtwos{\inputv}\bigr]}\\
&= 16\radiusb^{2\nbrlayers+2}\nbrlayers^2 (\constinput)^2\,,
\end{align*}
\endgroup
and, hence, 
$\constenv\leq 4\radiusb^{\nbrlayers+1} \nbrlayers\constinput$.
We can then conclude by putting this result back into the result of Step~2.

\emph{Step~4:} The first inequality in Theorem~\ref{robustbound} finally follows from consolidating the result of Step~3 and using the fact that $\level\in(0,1)$.

The second inequality follows from the first one and this derivation:
\begin{align*}
  &E_{(\outputE,\inputv)}\Bigl[\loss\bigl[\outputE-\functionTF{\inputv}\bigr]\Bigr]\\
  &\leq E_{(\outputE,\inputv)}\Bigl[\absB{\loss\bigl[\outputE-\functionTF{\inputv}\bigr]-\lossF{0}}\Bigr]\\
  &\leq E_{(\outputE,\inputv)}\Bigl[\constlip\absB{\outputE-\functionTF{\inputv}-0}\Bigr]\\
  &=\constlip E_{(\outputE,\inputv)}\Bigl[\absB{\outputE-\functionTF{\inputv}}\Bigr]\\
  &\leq \constlip \sqrt{E_{(\outputE,\inputv)}\Bigl[\absB{\outputE-\functionTF{\inputv}}^2\Bigr]}\\
&=\constlip\constnoise\,,
\end{align*}
where we use similar techniques as in the other parts of the proof.
\end{proof}

\section{Discussion}\label{sec:discussion}
Our statistical guarantees show that replacing the standard least-squares loss with a Lipschitz-continuous loss renders weight decay an effective method for regression for a broad spectrum of data.
This spectrum includes benign data (such as sub-Gaussian or bounded data) but also corrupted data (having outliers that are caused by an adversary or by other means).
More generally,
our results provide theoretical support for the use of robust loss functions in deep learning.

We have formulated our bounds for weight decay, because it is arguably the most popular type of regularization in view of its ability to avoid overfitting and accelerate computations \citep{Krizhevsky2012}.
But one can easily transfer our derivations to other types of regularization---as long as there are appropriate bounds for the Rademacher complexities.  

Some robust loss functions,
such as Huber and Cauchy loss,
involve an additional parameter: 
see Figure~\ref{fig:losses}.
Ideas for how to calibrate this parameter in practice can be found in~\citet{Chichignoud14} and \citet{Loh18}.

It is straightforward to generalize our results from empirical-risk minimizers to approximate empirical-risk minimizers.
Such generalizations take into account that minimizers can rarely be computed exactly.
But our theories do not apply to local minima:
this is a limitation that our paper has in common with most statistical literature on deep learning.


\ifarXiv
\acks{We thank Koosha Alaiemahabadi, Yannick D\"uren, Shih-Ting Huang, Mike Laszkiewicz, Nils M\"uller, Mahsa Taheri, and Fang Xie for the inspiring discussions.}
\fi



\appendix


\bibliography{Bibliography}

\def\bibarXiv[#1]{arXiv:#1}\def\bibAAAI{AAAI Conference on Artificial
  Intelligence}\def\bibAdvAP{Adv.\@ in Appl.\@
  Probab}\def\bibAnnHenriPoincare{Ann.\@ Henri
  Poincar\'e}\def\bibAnnRevStat{Annu.\@ Rev.\@ Stat.\@ Appl.}\def\bibAoS{Ann.\@
  Statist.}\def\bibBernoulli{Bernoulli}\def\bibBiometrika{Biometrika}\def\bibBiostatistics{Biostatistics}\def\bibBullLondonMS{Bull.\@
  Lond.\@ Math.\@ Soc.}\def\bibCSDA{Comput.\@ Statist.\@ Data
  Anal.}\def\bibEJS{Electron.\@ J.\@ Stat.}\def\bibIEEETIT{IEEE Trans.\@
  Inform.\@ Theory}\def\bibIsraelJM{Israel J.\@ Math.}\def\bibJASA{J.\@ Amer.\@
  Statist.\@ Assoc.}\def\bibJCGS{J.\@ Comput.\@ Graph.\@
  Statist.}\def\bibJMA{J.\@ Multivariate Anal.}\def\bibJMLR{J.\@ Mach.\@
  Learn.\@ Res.}\def\bibJRSSB{J.\@ R.\@ Stat.\@ Soc.\@ Ser.\@ B Stat.\@
  Methodol.}\def\bibJSPI{J.\@ Statist.\@ Plann.\@ Inference}\def\bibNIPS{Adv.\@
  Neural Inf.\@ Process.\@ Syst.}\def\bibML{Machine
  Learning}\def\bibPhysRev{Phys.\@ Rev.\@ D}\def\bibPLoS{PLoS Comput.\@
  Biol.}\def\bibPMLR{Proceedings of Machine Learning
  Research}\def\bibScience{Science}\def\bibStatScience{Statist.\@
  Sci.}\def\bibStatSinica{Statist.
  Sinica}\def\bibTechnometrics{Technometrics}\def\bibTest{Test}\def\bibTheoryProb{Theory
  Probab.\@ Appl.}
\begin{thebibliography}{43}
\providecommand{\natexlab}[1]{#1}
\providecommand{\url}[1]{\texttt{#1}}
\expandafter\ifx\csname urlstyle\endcsname\relax
  \providecommand{\doi}[1]{doi: #1}\else
  \providecommand{\doi}{doi: \begingroup \urlstyle{rm}\Url}\fi

\bibitem[Akhtar and Mian(2018)]{akhtar2018threat}
N.~Akhtar and A.~Mian.
\newblock Threat of adversarial attacks on deep learning in computer vision: a
  survey.
\newblock \emph{\bibarXiv[1801.00553]}, 2018.

\bibitem[Anthony and Bartlett(2009)]{Anthony09}
M.~Anthony and P.~Bartlett.
\newblock \emph{Neural network learning: theoretical foundations}.
\newblock {C}ambridge {U}niversity {P}ress, 2009.

\bibitem[Barron(2019)]{Barron19}
J.~Barron.
\newblock A general and adaptive robust loss function.
\newblock In \emph{Proc.\@ ICCV}, pages 4331--4339, 2019.

\bibitem[Bartlett(1998)]{Bartlett1998}
P.~Bartlett.
\newblock The sample complexity of pattern classification with neural networks:
  The size of the weights is more important than the size of the network.
\newblock \emph{IEEE Trans.\@ Inform.\@ Theory}, 44\penalty0 (2):\penalty0
  525--536, 1998.

\bibitem[Bartlett and Mendelson(2002)]{Bartlett02}
P.~Bartlett and S.~Mendelson.
\newblock {Rademacher} and {Gaussian} complexities: risk bounds and structural
  results.
\newblock \emph{J.\@ Mach.\@ Learn.\@ Res.\@}, 3:\penalty0 463--482, 2002.

\bibitem[Bartlett et~al.(2002)Bartlett, Boucheron, and Lugosi]{Bartlett02b}
P.~Bartlett, S.~Boucheron, and G.~Lugosi.
\newblock Model selection and error estimation.
\newblock \emph{Machine Learning}, 48:\penalty0 85--113, 2002.

\bibitem[Belagiannis et~al.(2015)Belagiannis, Rupprecht, Carneiro, and
  Navab]{Belagiannis15}
V.~Belagiannis, C.~Rupprecht, G.~Carneiro, and N.~Navab.
\newblock Robust optimization for deep regression.
\newblock In \emph{Proc.\@ ICCV}, 2015.

\bibitem[Boucheron et~al.(2016)Boucheron, Lugosi, and Massart]{Boucheron13}
S.~Boucheron, G.~Lugosi, and P.~Massart.
\newblock \emph{Concentration inequalities: a nonasymptotic theory of
  independence}.
\newblock Oxford {U}niversity {P}ress, 2016.

\bibitem[Chichignoud and Lederer(2014)]{Chichignoud14}
M.~Chichignoud and J.~Lederer.
\newblock A robust, adaptive {M}-estimator for pointwise estimation in
  heteroscedastic regression.
\newblock \emph{Bernoulli}, 20\penalty0 (3):\penalty0 1560--1599, 2014.

\bibitem[Durrett(2010)]{Durrett10}
R.~Durrett.
\newblock \emph{Probability: theory and examples}.
\newblock Cambridge {U}niversity {P}ress, fourth edition, 2010.

\bibitem[Golowich et~al.(2020)Golowich, Rakhlin, and Shamir]{Golowich17}
N.~Golowich, A.~Rakhlin, and O.~Shamir.
\newblock Size-independent sample complexity of neural networks.
\newblock \emph{Information and Inference}, 9\penalty0 (2):\penalty0 473--504,
  2020.

\bibitem[Haagerup(1981)]{Haagerup81}
U.~Haagerup.
\newblock The best constants in the {K}hintchine inequality.
\newblock \emph{Studia Math.}, 70:\penalty0 231--283, 1981.

\bibitem[Hahnloser(1998)]{10.1016/S0893-6080(98)00012-4}
R.~Hahnloser.
\newblock On the piecewise analysis of networks of linear threshold neurons.
\newblock \emph{Neural Networks}, 11\penalty0 (4):\penalty0 691–697, 1998.

\bibitem[Hampel et~al.(2011)Hampel, Ronchetti, Rousseeuw, and Stahel]{Hampel11}
F.~Hampel, E.~Ronchetti, P.~Rousseeuw, and W.~Stahel.
\newblock \emph{Robust statistics: the approach based on influence functions}.
\newblock John Wiley \& Sons, 2011.

\bibitem[Huber and Ronchetti(2009)]{Huber09}
P.~Huber and E.~Ronchetti.
\newblock \emph{Robust Statistics}.
\newblock John Wiley \& Sons, second edition, 2009.

\bibitem[Jiang et~al.(2018)Jiang, Zhou, Leung, Li, and Fei-Fei]{Jiang18}
L.~Jiang, Z.~Zhou, T.~Leung, L.-J. Li, and L.~Fei-Fei.
\newblock Mentornet: learning data-driven curriculum for very deep neural
  networks on corrupted labels.
\newblock In \emph{Proc.\@ ICML}, number~35, pages 2304--2313, 2018.

\bibitem[Koltchinskii(2001)]{Koltchinskii01}
V.~Koltchinskii.
\newblock Rademacher penalties and structural risk minimization.
\newblock \emph{\bibIEEETIT}, 47\penalty0 (5):\penalty0 1902--1914, 2001.

\bibitem[Koltchinskii and Panchenko(2002)]{Koltchinskii02}
V.~Koltchinskii and D.~Panchenko.
\newblock Empirical margin distributions and bounding the generalization error
  of combined classifiers.
\newblock \emph{\bibAoS}, 30\penalty0 (1):\penalty0 1--50, 2002.

\bibitem[Kos and Song(2017)]{kos2017delving}
J.~Kos and D.~Song.
\newblock Delving into adversarial attacks on deep policies.
\newblock ICLR Workshop, 2017.

\bibitem[Krizhevsky et~al.(2012)Krizhevsky, Sutskever, and
  Hinton]{Krizhevsky2012}
A.~Krizhevsky, I.~Sutskever, and G.~Hinton.
\newblock Imagenet classification with deep convolutional neural networks.
\newblock In \emph{\bibNIPS}, pages 1097--1105, 2012.

\bibitem[Krogh and Hertz(1991)]{Krogh92}
A.~Krogh and J.~Hertz.
\newblock A simple weight decay can improve generalization.
\newblock In \emph{\bibNIPS}, number~4, pages 950--957, 1991.

\bibitem[Kurakin et~al.(2016)Kurakin, Goodfellow, and
  Bengio]{kurakin2016adversarial}
A.~Kurakin, I.~Goodfellow, and S.~Bengio.
\newblock Adversarial examples in the physical world.
\newblock \emph{\bibarXiv[1607.02533]}, 2016.

\bibitem[Kurakin et~al.(2017)Kurakin, Ian~Goodfellow, and
  Bengio]{kurakin2016adversarial2}
A.~Kurakin, I.~Ian~Goodfellow, and S.~Bengio.
\newblock Adversarial machine learning at scale.
\newblock In \emph{Proc.\@ ICLR}, 2017.

\bibitem[Lab(2019)]{tksl}
Tencent Keen~Security Lab.
\newblock Experimental security research of tesla autopilot, 2019.

\bibitem[Lederer and {van de Geer}(2014)]{Lederer14}
J.~Lederer and S.~{van de Geer}.
\newblock New concentration inequalities for suprema of empirical processes.
\newblock \emph{\bibBernoulli}, 20\penalty0 (4):\penalty0 2020--2038, 2014.

\bibitem[Loh(2018)]{Loh18}
P.-L. Loh.
\newblock Scale calibration for high-dimensional robust regression.
\newblock \emph{\bibarXiv[1811.02096]}, 2018.

\bibitem[Madry et~al.(2017)Madry, Makelov, Schmidt, Tsipras, and
  Vladu]{aleks2017deep}
A.~Madry, A.~Makelov, L.~Schmidt, D.~Tsipras, and A.~Vladu.
\newblock Towards deep learning models resistant to adversarial attacks.
\newblock In \emph{Proc.\@ ICLR}, 2017.

\bibitem[McDiarmid(1989)]{McDiarmid89}
C.~McDiarmid.
\newblock On the method of bounded differences.
\newblock \emph{Surv.\@ Comb.}, 141\penalty0 (1):\penalty0 148--188, 1989.

\bibitem[Mohri et~al.(2018)Mohri, Rostamizadeh, and Talwalkar]{Mohri18}
M.~Mohri, A.~Rostamizadeh, and A.~Talwalkar.
\newblock \emph{Foundations of machine learning}.
\newblock MIT Press, second edition, 2018.

\bibitem[Moosavi-Dezfooli et~al.(2017)Moosavi-Dezfooli, Fawzi, Fawzi, and
  Frossard]{Moosavi17}
S.-M. Moosavi-Dezfooli, A.~Fawzi, O.~Fawzi, and P.~Frossard.
\newblock Universal adversarial perturbations.
\newblock In \emph{IEEE Int.\@ Conf.\@ Comput.\@ Vis.\@ Pattern Recognit.},
  pages 1765--1773, 2017.

\bibitem[Neyshabur et~al.(2015)Neyshabur, Tomioka, and Srebro]{Neyshabur2015}
B.~Neyshabur, R.~Tomioka, and N.~Srebro.
\newblock Norm-based capacity control in neural networks.
\newblock In \emph{Proc.\@ COLT}, number~28, pages 1376--1401, 2015.

\bibitem[Papernot et~al.(2015)Papernot, McDaniel, Wu, Jha, and
  Swami]{papernot2015distillation}
N.~Papernot, P.~McDaniel, X.~Wu, S.~Jha, and A.~Swami.
\newblock Distillation as a defense to adversarial perturbations against deep
  neural networks.
\newblock \emph{\bibarXiv[1511.04508]}, 2015.

\bibitem[Salinas and Abbott(1996)]{Salinas11956}
E.~Salinas and L.~Abbott.
\newblock A model of multiplicative neural responses in parietal cortex.
\newblock \emph{Proc.\@ Nat.\@ Acad.\@ Sci. USA}, 93\penalty0 (21):\penalty0
  11956--11961, 1996.

\bibitem[Salman et~al.(2019)Salman, Yang, Li, Zhang, Zhang, Razenshteyn, and
  Bubeck]{salman2019provably}
H.~Salman, G.~Yang, J.~Li, P.~Zhang, H.~Zhang, I.~Razenshteyn, and S.~Bubeck.
\newblock Provably robust deep learning via adversarially trained smoothed
  classifiers.
\newblock In \emph{\bibNIPS}, number~32, 2019.

\bibitem[Schmidt-Hieber(2020)]{Hieber2017}
J.~Schmidt-Hieber.
\newblock Nonparametric regression using deep neural networks with {ReLU}
  activation function.
\newblock \emph{\bibAoS}, 48\penalty0 (4):\penalty0 1875--1897, 2020.

\bibitem[Sharif et~al.(2016)Sharif, Bhagavatula, Bauer, and
  Reiter]{10.1145/2976749.2978392}
M.~Sharif, S.~Bhagavatula, L.~Bauer, and M.~Reiter.
\newblock Accessorize to a crime: real and stealthy attacks on state-of-the-art
  face recognition.
\newblock In \emph{Proc.\@ CCS}, pages 1528--1540, 2016.

\bibitem[Stigler(2010)]{Stigler2010}
S.~Stigler.
\newblock The changing history of robustness.
\newblock \emph{Amer.\@ Statist.}, 64\penalty0 (4):\penalty0 277--281, 2010.

\bibitem[Taheri et~al.(2020)Taheri, Xie, and Lederer]{Taheri20}
M.~Taheri, F.~Xie, and J.~Lederer.
\newblock Statistical guarantees for regularized neural networks.
\newblock \emph{arXiv:2006.00294}, 2020.

\bibitem[Tram\'er et~al.(2017)Tram\'er, Kurakin, Papernot, Goodfellow, Boneh,
  and McDaniel]{tramr2017ensemble}
F.~Tram\'er, A.~Kurakin, N.~Papernot, I.~Goodfellow, D.~Boneh, and P.~McDaniel.
\newblock Ensemble adversarial training: attacks and defenses.
\newblock \emph{\bibarXiv[1705.07204]}, 2017.

\bibitem[Wang and Yu(2019)]{wang2019direct}
H.~Wang and C.-N. Yu.
\newblock A direct approach to robust deep learning using adversarial networks.
\newblock In \emph{Proc.\@ ICLR}, 2019.

\bibitem[Wang et~al.(2018)Wang, Gu, Mehta, Zhao, and Bernal]{wang2018robust}
T.~Wang, Y.~Gu, D.~Mehta, X.~Zhao, and E.~Bernal.
\newblock Towards robust deep neural networks.
\newblock \emph{\bibarXiv[1810.11726]}, 2018.

\bibitem[Wang et~al.(2016)Wang, Chang, Yang, Liu, and Huang]{Wang16}
Z.~Wang, S.~Chang, Y.~Yang, D.~Liu, and T.~Huang.
\newblock Studying very low resolution recognition using deep networks.
\newblock In \emph{Proc.\@ CVPR}, pages 4792--4800, 2016.

\bibitem[{Yuan} et~al.(2019){Yuan}, {He}, {Zhu}, and {Li}]{8611298}
X.~{Yuan}, P.~{He}, Q.~{Zhu}, and X.~{Li}.
\newblock Adversarial examples: attacks and defenses for deep learning.
\newblock \emph{IEEE Trans.\@ Neural Netw.\@ Learn.\@ Syst.\@}, 30\penalty0
  (9):\penalty0 2805--2824, 2019.

\end{thebibliography}

\end{document}